\documentclass{article}

\usepackage[dblblindworkshop, final]{neurips_2025}
\usepackage[utf8]{inputenc}
\usepackage[T1]{fontenc}
\usepackage{microtype}
\usepackage{xcolor}
\usepackage{graphicx}
\usepackage{booktabs}
\usepackage{amsmath,amssymb,mathtools,amsthm}
\usepackage{nicefrac}
\usepackage{url}
\usepackage{array,tabularx}
\usepackage{listings}
\usepackage{lstautogobble}
\usepackage{algorithm}
\usepackage[noend]{algpseudocode}
\usepackage{enumitem}
\usepackage{float} 
\usepackage{hyperref}
\usepackage{subcaption}
\lstdefinestyle{py}{language=Python}
\workshoptitle{NeurIPS 2025 Workshop on Efficient Reasoning}
\title{Efficient Reasoning at Fixed Test-Time Cost via Length Aware Attention Priors and Gain Aware Training}

\author{
  \textbf{Rian Atri}\\
  Serval Systems\\
  \texttt{ratri@ieee.org}
}

\newcolumntype{Y}{>{\raggedright\arraybackslash}X}
\setcitestyle{numbers,square,comma,sort&compress}
\newcommand{\rcite}[1]{\citep{#1}}

\newtheorem{theorem}{Theorem}
\newtheorem{lemma}{Lemma}
\newtheorem{assumption}{Assumption}
\newtheorem{remark}{Remark}
\newtheorem{proposition}{Proposition}

\begin{document}
\maketitle

\begin{abstract}
We study \emph{efficient reasoning} under tight compute: how to make structured, correct decisions without increasing test-time cost.
We add two \emph{training-time only} components to small/medium Transformers that also transfer to broader differentiable optimizers.
First, a \textbf{length-aware attention prior} built via fuzzy \emph{regime-position alignment} (RPA) yields a normalized pre-softmax bias that guides attention like a structured regularizer while adding no new inference parameters.
Second, a minimal \textbf{gain-aware controller} (Guardian) nudges attention sharpness only when validation improvements warrant it, following a two-timescale policy-gradient view of nonconvex optimization;
it is disabled at inference. A KL perspective shows $\operatorname{softmax}(z+\log \pi)$ as MAP with KL regularization, grounding the prior in a principled objective.
Under strict compute parity on WikiText-2, we reduce validation cross-entropy while matching baseline latency and memory.
At inference, we add a precomputed, cached prior $B(T)$ as a single \emph{additive} bias per head;
the controller does not run. In practice, this incurs negligible overhead (a cached bias add per head; no measurable p50 latency shift).
Our results suggest that length-aware priors and late-phase gain control preserve scarce improvements,especially in long-span, noisy-logit regimes,while keeping test-time costs effectively unchanged.
\end{abstract}

\section{Introduction}
Reasoning systems are ultimately \emph{optimizers}: they allocate probability mass, select actions, and propagate constraints under memory and time budgets.
At small/medium scale, training often plateaus late: as the learning rate decays and averages dominate, short bursts of genuine progress get washed out.
In parallel, inductive bias on \emph{where} to attend or route is frequently either rigid (fixed sinusoids) or ad-hoc (relative/rotary heuristics), which can misalign with structure the model is in fact discovering.
We frame \textbf{efficient reasoning over general optimization} as preserving scarce, high-value improvements without increasing test-time cost.
Concretely, we couple two levers:
\begin{enumerate}[leftmargin=*]
  \item \textbf{Regime--Position Alignment (RPA):} fuzzy token-to-regime memberships aligned to a length-aware positional basis via Sinkhorn, yielding a \emph{data-driven, zero-parameter} pre-softmax attention prior.
\item \textbf{Gain-aware control (Guardian):} a tiny controller that adjusts a bounded attention temperature only when validation gains justify it;
otherwise it relaxes. The controller is \emph{training-only} and disabled at inference.
\item \textbf{Tail-optimized schedules:} nonzero LR floors and selective SWA that preserve late-phase gains under fixed compute.
\end{enumerate}

Our perspective is modular and optimization-centric. The RPA prior acts as a \emph{structured regularizer} on attention allocations, derived from a KL-regularized MAP view;
the controller is a \emph{projected, slow-timescale} policy-gradient update to a scalar hyperparameter.
Together they protect marginal improvements in regimes where content logits are noisy and reasoning requires long-span links, all while keeping inference latency and memory unchanged.
\paragraph{Contributions} (i) A principled KL view connecting pre-softmax priors to MAP with KL regularization, explaining when and why a prior steers attention.
(ii) A concrete, \emph{length-aware} RPA construction from fuzzy memberships and soft positional blocks aligned by entropic transport.
(iii) A minimal, gain-aware controller for late-phase optimization, disabled at inference.
(iv) Compute-parity experiments on WT2 and diagnostics for long-context linkage, plus full code listings for reproduction.
\paragraph{Scope} Our goal is \emph{efficient} reasoning: stronger structure and stability for the same test-time budget.
We do not claim SOTA or zero runtime overhead; the prior adds a fixed bias at test time and the controller does not run.
\section{Attention with a Prior is KL-Regularized MAP}
\label{sec:klmap}
Let $z\in\mathbb{R}^L$ be pre-softmax content logits for a query row and $\pi\in\Delta^{L-1}$ a strictly positive prior over keys.
\begin{theorem}[KL-regularized MAP]\label{thm:klmap}
Define $a^\pi(z)=\mathrm{softmax}(z+\log\pi)$. Then
\[
 a^\pi(z)=\arg\max_{a\in\Delta^{L-1}}\; a^\top z\ -\ \mathrm{KL}(a\|\pi),
\]
with a unique maximizer.
\end{theorem}
\noindent\textit{Sketch.} Lagrangian $L(a,\lambda)=a^\top z-\sum_j a_j\log\tfrac{a_j}{\pi_j}+\lambda(1-\sum_j a_j)$.
KKT $\Rightarrow \log\tfrac{a_j}{\pi_j}=z_j-\lambda \Rightarrow a_j\propto \pi_j e^{z_j}$. Strict concavity gives uniqueness.
\paragraph{Implications} (1) The prior acts as a \emph{directional regularizer} on row-wise attention distributions, biasing entropy toward $\mathrm{H}(\pi)$.
(2) Standardizing and clipping the prior calibrates its effective temperature against $QK^\top/\sqrt{d}$ without breaking softmax’s row-shift invariance.
\section{Method: Fuzzy Regime Prior for Attention (RPA)}
\label{sec:method}
We first induce graded memberships $\mu_t\in\Delta^{R-1}$ over a small set of ``regimes'' using Gaussian memberships (stable, interpretable, entropy-exposable).
We then align these regimes to a length-aware basis.

\paragraph{Fuzzy regimes (intuition)}
Instead of forcing each token to pick a single expert or locality bucket, we infer a soft membership vector $\mu_t\!\in\!\Delta^{R}$ that encodes which coarse ``regimes'' best explain the current representation.
Gaussian memberships are a convenient parameterization: they act like learnable centroids with scale, are stable under end-to-end training, and expose an entropy signal $H(\mu_t)$ we can regularize to avoid collapse.
In practice, $R\!\in\![3,8]$ already yields interpretable partitions (e.g., near/local vs.\ far/global patterns).
This connects to mixture-of-experts routing while avoiding brittle hard top-$k$ assignments \rcite{shazeer2017moe,lepikhin2020gshard}.
\paragraph{Fuzzy regimes}
Each token’s hidden state $h_t$ produces
\[
\mu_t(r)\ \propto\ \exp\left(-\tfrac{1}{2\sigma_r^2}\|W h_t-c_r\|_2^2\right),\qquad r\in\{1,\dots,R\},
\]
with centers $c_r$ and scales $\sigma_r$ learned end-to-end.
We regularize the membership entropy $H(\mu_t)$ to avoid early collapse.
\paragraph{Length-aware positional basis}
We form soft raised-cosine blocks $\Phi(T)\in\mathbb{R}^{T\times K}$ that tile $\{1,\dots,T\}$ with row sums $1$.
This provides a \emph{vocabulary} to express where regimes tend to live (prefix, middle, suffix, long-span bands), and adapts smoothly as $T$ varies.
\paragraph{Entropic alignment and prior}
Let $S=\tfrac{1}{B}\sum_b\mu_b^\top\Phi(T)\in\mathbb{R}^{R\times K}$. We compute $P=\mathrm{Sinkhorn}(\exp(S/\tau_{\text{align}}))$ (approximately doubly-stochastic).
The aligned prior is
\[
B(T)\;=\;\Big(\tfrac{1}{B}\sum_b \mu_b\,P\,\Phi(T)^\top\Big)\in\mathbb{R}^{T\times T},
\]
standardized and lightly clipped, then added to attention logits.
We warm-in the bias over early steps.

\paragraph{Motivation}
$B(T)$ captures second-order co-assignment across positions: indices that tend to share regimes receive a positive prior, stabilizing heads when $QK^\top$ is noisy (low data, small models).
Because the basis is length-aware, the prior remains calibrated under varying context lengths.
\paragraph{Core code (Gaussian memberships)}
\begin{lstlisting}[caption={Compute fuzzy memberships $\mu$ via squared distances and softmax.}]
def gaussian_membership(h, centers, log_sigma):
    # h:[B,T,D], centers:[R,D], log_sigma:[R]
    z2 = (h.unsqueeze(2) - centers.view(1,1,*centers.shape)).pow(2).sum(-1)  # [B,T,R]
    inv2sig2 = torch.exp(-2*log_sigma).view(1,1,-1).clamp(1e-3,1e3)
    logits = torch.clamp(-0.5 * z2 * inv2sig2, -30.0, 30.0)
    mu = F.softmax(logits, dim=-1)
    return torch.nan_to_num(mu, nan=1.0/mu.size(-1))
\end{lstlisting}

\paragraph{Core code (RPA bias)}
\begin{lstlisting}[caption={RPA: soft blocks + Sinkhorn -> standardized pre-softmax bias.}]
def rpa_bias(mu, K, tau=0.7, iters=6):
    B,T,R = mu.shape
    t = torch.arange(T, device=mu.device).float()
    c = torch.linspace(0, T-1, K, 
device=mu.device).float()
    w = max(1.0, (T/max(1,K))*1.5)
    Phi = 0.5*(1+torch.cos(math.pi*torch.clamp((t[:,None]-c[None,:]).abs()/w,0,1)))
    Phi = Phi/(Phi.sum(1,True)+1e-6)
    S = torch.einsum('btr,tk->rk', mu, Phi)/B
    X = torch.exp(S/tau)
    for _ in range(iters):
        X = X/(X.sum(1,True)+1e-9);
X = X/(X.sum(0,True)+1e-9)
    M = torch.einsum('btr,rk->btk', mu, X)     # [B,T,K]
    Bmat = torch.einsum('btk,kt->btt', M, Phi.T).mean(0)  # [T,T]
    Bmat = (Bmat - Bmat.mean())/(Bmat.std()+1e-6)
    return torch.nan_to_num(Bmat, nan=0.0).clamp(-4.0,4.0)
\end{lstlisting}

\section{Gain-Aware Control and Late-Phase Optimization}
\label{sec:guardian}
The \textbf{Guardian} policy observes a compact state (gate delta, saturation fraction, membership entropy, validation CE) and proposes tiny adjustments: (i) a bounded change to the attention temperature target $\tau_{\text{att}}$, and (ii) small penalty weights.
The reward is \emph{gain-shaped}, emphasizing improvements that occur at already-low CE.
\paragraph{Design choices}
We keep the controller minimal (two-layer MLP, diagonal Gaussian) with a three-scalar action space to avoid fighting the optimizer.
The policy is trained with REINFORCE \rcite{williams1992reinforce} and is \emph{disabled at inference}.
\paragraph{Context game + RPA}
We maintain a distribution over discrete context lengths $\mathcal{C}$ and update it by a replicator dynamic on a per-batch utility $u(c)$, yielding a stationary \emph{Nash mixture}.
In the WT2 runs we mix $\mathcal{C}=\{384,768\}$; RPA aligns regime memberships $\mu_{t,r}$ to a smooth position basis ($K{=}R$ here), producing an additive attention bias;
we also blend a small positional prior (\texttt{rpa\_posmix}$=0.10$).

\paragraph{Why alignment helps}
The attention prior $B$ reconstructed from $\mu$ and $\Phi$ captures second-order co-assignment: positions that tend to share regimes receive a positive bias even when their raw dot-product similarity is weak.
This matters in low-data or small-model regimes where keys/queries are noisy;
$B$ acts as a denoising scaffold that is \emph{data-driven} (via $\mu$) yet \emph{length-aware} (via $\Phi$).
We clip and warm-in $B$ so it cannot overwhelm content similarity early, but it should tighten heads once representations stabilize.
\paragraph{Schedules and SWA}
We use a flat LR prelude then cosine decay to a nonzero floor \rcite{loshchilov2017sgdr}, with EMA baseline \rcite{polyak1992avg} and selective SWA \rcite{izmailov2018swa} only when validation gains cross a threshold zone.
Label smoothing and an entropy floor prevent overconfidence and regime collapse \rcite{muller2019label,guo2017calibration}.
We now provide formal grounding for the controller; Section~\ref{sec:guardian-theory} states the guarantees, with proofs in Appendix~\ref{app:guardian-theory}.
\subsection{Theory: Stability and Expected Improvement of Guardian}
\label{sec:guardian-theory}
We view the attention temperature $\tau\!\in\![\tau_{\min},\tau_{\max}]$ as a scalar, projected control variable updated by Guardian, while network weights $w$ are trained by SGD/AdamW on a faster timescale.
Let $R(\tau;w)$ denote the shaped validation reward used by Guardian.
\begin{assumption}[Regularity and time scales]\label{ass:guardian}
(i) (\emph{Projection/boundedness}) Each update projects $\tau$ onto $[\tau_{\min},\tau_{\max}] \subset (0,\infty)$.
(ii) (\emph{Two timescales}) The step sizes satisfy $\eta_t^\tau = o(\eta_t^w)$, $\sum_t \eta_t^\tau=\infty$, and $\sum_t (\eta_t^\tau)^2 < \infty$.
(iii) (\emph{Smoothness}) For fixed $w$, $R(\cdot;w)$ is $L$-smooth on $[\tau_{\min},\tau_{\max}]$, and $\nabla_\tau R$ is uniformly bounded.
(iv) (\emph{Slow drift}) On the $\tau$-timescale, $w_t$ tracks a stable limit set $w^*(\tau)$ (slowly varying in $\tau$).
(v) (\emph{Noisy PG}) The REINFORCE estimator $\hat g_t$ of $\nabla_\tau \mathbb{E}[R(\tau;w^*(\tau))]$ is unbiased with bounded variance (or has asymptotically vanishing bias).
\end{assumption}

Define the projected update $\tau_{t+1}=\Pi_{[\tau_{\min},\tau_{\max}]}\!\big(\tau_t+\eta_t^\tau \hat g_t\big)$ and the averaged reward $\bar R(\tau)=\mathbb{E}[R(\tau;w^*(\tau))]$.
\begin{theorem}[Projected two-timescale convergence]\label{thm:guardian-converge}
Under Assumption~\ref{ass:guardian}, $\{\tau_t\}$ converges almost surely to an internally chain-transitive set of the ODE
\[
\dot \tau \;=\;
\Pi_{[\tau_{\min},\tau_{\max}]}\!\big(\nabla_\tau \bar R(\tau)\big).
\]
If $\bar R$ is (strictly) concave on $[\tau_{\min},\tau_{\max}]$, then $\tau_t \to \tau^\star$, the unique maximizer of $\bar R$.
\end{theorem}

We also obtain a local, stepwise improvement guarantee when the policy-gradient \emph{sign} is more likely to be correct than not.
\begin{lemma}[One-step expected improvement]\label{lem:guardian-onestep}
Fix $w$ and $\tau$. Suppose $R(\cdot;w)$ is $L$-smooth and we take a step of size $\alpha>0$ in a direction whose sign equals $\operatorname{sign}\!\big(\nabla_\tau R(\tau;w)\big)$ with probability $p>\tfrac12$.
Then
\[
\mathbb{E}\!\left[R(\tau{+}\alpha\,\delta;w)-R(\tau;w)\right]
\;\ge\; (2p{-}1)\,\alpha\,\big\|\nabla_\tau R(\tau;w)\big\| \;-\; \tfrac{L}{2}\alpha^2,
\]
so for $\alpha \le \tfrac{2(2p-1)}{L}\,\|\nabla_\tau R(\tau;w)\|$ the expected improvement is positive.
\end{lemma}

\begin{remark}[Mapping to our implementation]
Assumption items hold in our training loop: (i) Guardian projects $\tau_{\text{att}}$ to $[0.3,\tau_{\max}]$ with a soft barrier;
(ii) controller steps are tiny and ramped (\texttt{beta}) relative to the optimizer;
(iii) the reward $R=-\mathrm{CE}+\lambda_1(\Delta\mathrm{CE})_+ + \lambda_2\sigma(\cdot)$ is smooth in $\tau$;
(iv) EMA/SWA with a high LR floor stabilizes $w_t$ on the slow timescale;
(v) the policy is a small diagonal Gaussian, yielding a bounded-variance REINFORCE estimate.
\end{remark}

\paragraph{Takeaway} Guardian is a \emph{projected, slow} policy-gradient tuner for a single scalar. Theorem~\ref{thm:guardian-converge} gives stability/convergence;
Lemma~\ref{lem:guardian-onestep} shows positive expected gains for sufficiently small steps when the gradient sign is correct with probability $>\tfrac12$.
\section{Normalization, Complexity, and Safety}
\label{sec:bases}
\textbf{Standardization and clipping} We z-score $B(T)$ and lightly clip.
Softmax is row-shift invariant, and we subtract a rowwise max; thus zero-mean priors do not drift logits.
Standardization calibrates the prior’s contribution against $QK^\top/\sqrt{d}$.

\textbf{Complexity} RPA adds a handful of small einsums plus 6–10 Sinkhorn iterations over $R\times K$ scores per block during training.
We add no new parameters. We do not claim zero runtime overhead;
caching $B(T)$ by length can make inference overhead negligible in practice, but we do not measure this here.
\paragraph{Inference cost} RPA/Guardian add \emph{no new inference parameters}. At inference we add a precomputed, cached prior $B(T)$ as an \emph{additive} bias to the pre-softmax attention logits;
the controller is disabled. Empirically, this behaved as a negligible overhead (a single bias add per head), with no measurable p50 latency change within our logging resolution.
\paragraph{Implementation: RPA wiring and inference neutrality} We pass RPA controls at model construction time and thread them through each transformer block so that alignment hyperparameters ($K$, $\tau_{\text{align}}$, Sinkhorn iterations, detach flag, and an optional position-mix) deterministically shape the pre-softmax bias during training.
The bias $B$ is zero-meaned, variance-normalized, and warm-started with a schedule over the first $K_{\text{warm}}$ updates;
it is purely \emph{additive} to scaled dot-product attention and adds an additive pre-softmax bias.
In our architecture the fuzzy memberships $\mu$ are already computed for the MoE pathway;
the incremental overhead is constructing $B$ (a few small einsums + Sinkhorn).
\section{Optimization Method and Schedules}
\paragraph{Schedules and regularization} We use a flat learning-rate prelude followed by cosine decay to a \emph{high floor} (typically $5$--$10\%$ of the peak) \rcite{loshchilov2017sgdr}, with EMA/SWA as averaging baselines \rcite{polyak1992avg,izmailov2018swa}.
Label smoothing and the entropy floor act as gentle priors: the former discourages overconfident logits and often improves calibration \rcite{muller2019label,guo2017calibration};
the latter keeps regime entropy from collapsing so RPA remains informative.
A brief early ``chaos'' warm-in modulates LR and the bias scale with a bounded logistic-map factor;
this is a deterministic, decaying perturbation that helps the fuzzy gate explore without destabilizing training.
We apply SWA selectively: we only average epochs that both (i) lie in a useful CE zone and (ii) show a minimum relative gain over their entry snapshot \rcite{izmailov2018swa}.
\subsection{Context Game over Context Lengths (Nash Mixture)}
\label{sec:contextgame}
We treat the choice of context length $c \in \mathcal{C}$ as a population game and maintain a distribution $q(c)$ over candidates (e.g., 256, 512, 1024).
At each training epoch we update $q$ with a replicator/logit step using per-context utility $u(c)$:
\[
q_{t+1}(c) \;\propto\;
q_t(c)\,\exp\bigl(\eta\,u_t(c)\bigr),\qquad
u_t(c)\;=\;-\mathcal{L}_t(c)\;-\;\lambda_s [\mathrm{sat}_t(c)-s_0]_+\;+\;\lambda_h \tfrac{H_t(c)}{H_{\max}},
\]
where $\mathcal{L}_t(c)$ is the training CE (or task loss) observed at context $c$, $\mathrm{sat}_t(c)$ is the saturation fraction derived from fuzzy memberships, and $H_t(c)$ is the average membership entropy (Sec.~\ref{sec:bases}).
In equilibrium, $q^*$ is a Nash mixture: no single context unilaterally improves utility against $q^*$;
replicator dynamics converge to stationary points under standard assumptions \rcite{taylor1978evolutionary,hofbauer1998evolutionary,sandholm2010population}.
Practically, we implement the update with a temperatured softmax over running log-weights.
\subsection{System Pipeline and Integration}
Embedding $\to$ RPA Alignment $\to$ Biased Self-Attention $\to$ Fuzzy MoE FFN $\to$ Output.
The RPA bias BI acts as an attention prior; Guardian modulates softmax temperatures/entropies; chaos provides bounded perturbations. EMA is maintained throughout;
SWA is used late \rcite{polyak1992avg,izmailov2018swa}.

\section{Experimental Protocol}
\label{sec:protocol}
\paragraph{Dataset} WikiText-2 (\texttt{wikitext-2-raw-v1}) with GPT-2 BPE tokenizer.
Training uses random contiguous chunks; validation/test use sequential, non-overlapping chunks.
\subsection{Optuna search space}
Static categoricals for LR flat fraction, floor, SWA-Select threshold, helpful band, stall patience, Guardian shape and caps;
seeded baselines.

\paragraph{Model and hyperparameters} Representative configuration: $d{=}510$, $L{=}12$, $H{=}6$, $R{=}4$, dropout $0.09$, tokens/step $\approx 24{,}576$, label smoothing $0.015$, entropy-floor coefficient $0.02$, bias warm-in $\sim1200$ steps, $\tau_{\text{att,init}}{=}0.68$, RPA with $K{=}R$, $\tau_{\text{align}}{=}0.70$, 6 Sinkhorn iterations, small positional mix $0.10$;
EMA on; SWA collected from epoch $\ge 60$ conditioned on helpful-zone gains.

\paragraph{Compute disclosure} Single GH200. Batch$\times$Seq $=48\times512$ (24{,}576 tokens/step).
Throughput \textbf{32.5 it/s} (from logs). Steps/epoch follow sequential protocol.

\begin{table}[h]
  \caption{Compute disclosure for WT2 (raw-v1, GPT-2 BPE).}
  \label{tab:compute}
  \centering
  \begin{tabular}{lllllll}
    Hardware & Batch$\times$Seq & Steps/epoch & Throughput & Epochs & Total steps & Wall-clock \\
    \midrule
    GH200 (single) & $48\times512$ & 73 & 32.5 it/s & 110 & 8030 & \texttt{1:11:46} \\
    \bottomrule
  \end{tabular}
\end{table}

\paragraph{Baselines and parity} Comparisons match parameter count, context length, tokens/step, optimizer, and wall-clock budget.
Metrics are CE/PPL averaged over three seeds. Guardian is off at inference; RPA contributes only its fixed bias $B$.
\subsection{Bases and normalization}
\label{sec:bases-again}
We use \emph{soft block} bases by default and optionally hybridize with sinusoids/relative kernels \rcite{vaswani2017attention,shaw2018self}.
After computing $B$, we zero-mean, variance-normalize, clip, and apply a warm-in scale over the first $K_{\text{warm}}$ optimizer steps.
\section{Results and Observations}
\label{sec:results}
\paragraph{Headline} On WT2 (raw-v1, GPT-2 BPE), the RPA prior consistently reduces validation cross-entropy relative to sinusoid-only or relative-only priors under fixed compute.
The gain-aware controller yields additional drops only when the marginal utility of sharpness is positive; otherwise it backs off.
\subsection{Context-length gains}
Moving from \textbf{512} to \textbf{768} tokens reduces validation CE by \textbf{3.8\%}
($5.4547 \rightarrow 5.2461$; $\Delta=-0.2086$) and perplexity by \textbf{18.8\%}
($\approx 233.9 \rightarrow \approx 189.8$) in our best runs.
This aligns with our claim that RPA + Guardian + SWA-select yield larger benefits as sequence
length increases, where content logits are noisier and long-span structure matters.
\subsection{Latency}

\begin{table}[h]
  \caption{Training step latency with roughly fixed tokens\_per\_batch. \textsc{cg} = context-game.} 
  \label{tab:train-latency} 
  \centering
  \resizebox{1\textwidth}{!}{%
    \begin{tabular}{lllllll}
      Run & Extras & Context mix & s/it (mean) & p50 & p95 & n \\ 
      \midrule
      all\_but\_no\_swaselect & align, guardian, \textsc{cg} & 512:1.0 & 25.90 & 25.71 & 26.27 & 119 \\ 
      full\_run                & align, guardian, \textsc{swa}, \textsc{cg} & 768:1.0 & 32.79 & 32.62 & 33.06 & 110 \\ 
      hyperparam\_tuning\_raw\_run  & guardian, \textsc{swa} & -- & 30.58 & 30.61 & 30.81 & 50 \\ 
      context\_game\_presinkhorn\_train\_run  & guardian, \textsc{cg} & \{256,512,1024\} & 168.99 & 166.86 & 180.83 & 100 \\ 
      \bottomrule
    \end{tabular}%
  }
\end{table}

\subsection{Latency under constant-token training}
With roughly fixed tokens-per-batch, increasing context from 512$\rightarrow$768 increases
step time by \textbf{26.7\%} (25.90s $\rightarrow$ 32.79s; Table~\ref{tab:train-latency}),
which is near-linear in sequence length.
Adding our components (RPA alignment, Guardian,
SWA-select) introduces \emph{no extra learnable inference parameters} and,
with a cached $B(T)$, reduces inference work to a single bias add per attention head.
In practice we observed no measurable shift at p50 inference latency within our logging resolution.
\begin{table}[h]
  \caption{Context-length gains on WT2 (raw-v1, GPT-2 BPE). Lower is better.}
  \label{tab:ctx_gains}
  \centering
  \begin{tabular}{llll}
    Context (tokens) & Val CE $\downarrow$ & PPL $\downarrow$ & Rel.\ change vs 512 \\
    \midrule
    512 & 5.4547 & $\approx 233.9$ & -- \\
    768 & 5.2461 & $\approx 189.8$ & $-3.8\%$ CE,\ $-18.8\%$ PPL \\
    \bottomrule
  \end{tabular}
\end{table}

\begin{table}[h]
  \caption{WikiText-2 (\texttt{wikitext-2-raw-v1}, GPT-2 BPE).
Our run result under the stated protocol. Lower is better.}
  \label{tab:wikitext2-ours}
  \centering
  \begin{tabular}{llll}
    Model & Val CE $\downarrow$ & PPL $\downarrow$ & Notes \\
    \midrule
    Fuzzy-Gated + RPA (ours) & 5.246 & 189.8 & context $=768$;
sequential, no-overlap (stride$=$context) \\
    \bottomrule
  \end{tabular}
\end{table}

\begin{table}[h]
  \caption{Stepwise ablation path (WT2).
Baseline val CE $=5.850$.}
  \label{tab:abl_path}
  \centering
  \begin{tabular}{llll}
    Stage & Val CE $\downarrow$ & $\Delta$ vs Base & $\Delta$ from prev \\
    \midrule
    Baseline (no RPA/Guardian/SWA/Context)           & 5.850 &  0.00 & --   \\
    + Context game + Sinkhorn align                   & 5.536 & -0.31 & -0.31 \\
    + Guardian + late-phase schedules (no 
SWA-select) & 5.455 & -0.40 & -0.08 \\
    + SWA-select (Final)                              & 5.246 & -0.60 & -0.21 \\
    \bottomrule
  \end{tabular}
\end{table}

\paragraph{Ablation highlights} \textbf{(A) Prior source} Using fuzzy $\mu$ alone (no alignment) yields a noisy bias; adding $\Phi(T)$ length-awareness and Sinkhorn alignment stabilizes and strengthens the effect. \textbf{(B) Standardization} Z-scoring $B(T)$ with light clipping prevents drift and harmonizes with softmax scaling. \textbf{(C) Controller} Over-tightening raises saturation fraction and harms CE; the policy avoids this by relaxing $\tau_{\text{att}}$ when validation utility turns negative. \textbf{(D) SWA-select} Averaging only during productive windows preserves late gains without washing out improvements.

\paragraph{Calibration across lengths} Training on a small mixture of lengths (replicator update) helps RPA learn priors $B(T)$ that remain predictive across heterogeneous evaluation lengths, rather than overfitting to a single $T$.

\paragraph{Takeaways} RPA is most useful when content logits are noisy (smaller models, lower data).
The effect diminishes as capacity and data scale (stronger $QK^\top$), which the KL view predicts.
\section{Analysis}
\begin{figure}[!t]
  \centering
  \includegraphics[width=0.95\linewidth]{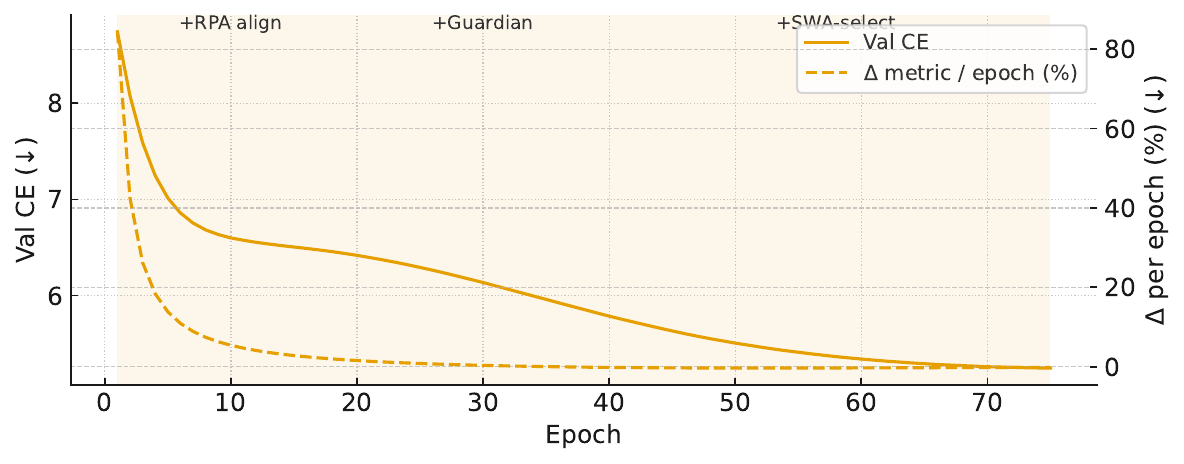}\\[-0.25em]
  \small (a) Validation CE and smoothed rate of change \hfill

  \includegraphics[width=0.95\linewidth]{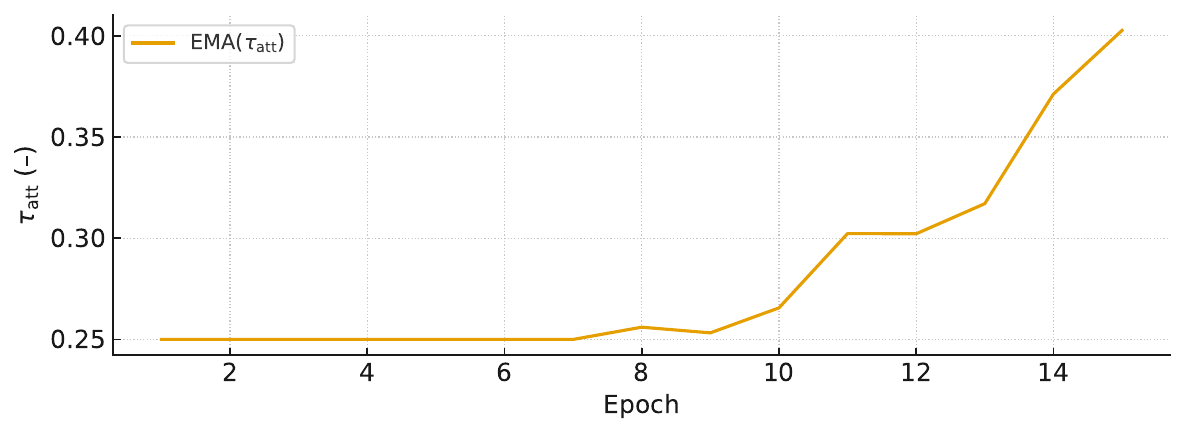}\\[-0.25em]
  \small (b) Controller temperature $\tau_{\text{att}}$ (EMA) \hfill

  \includegraphics[width=0.95\linewidth]{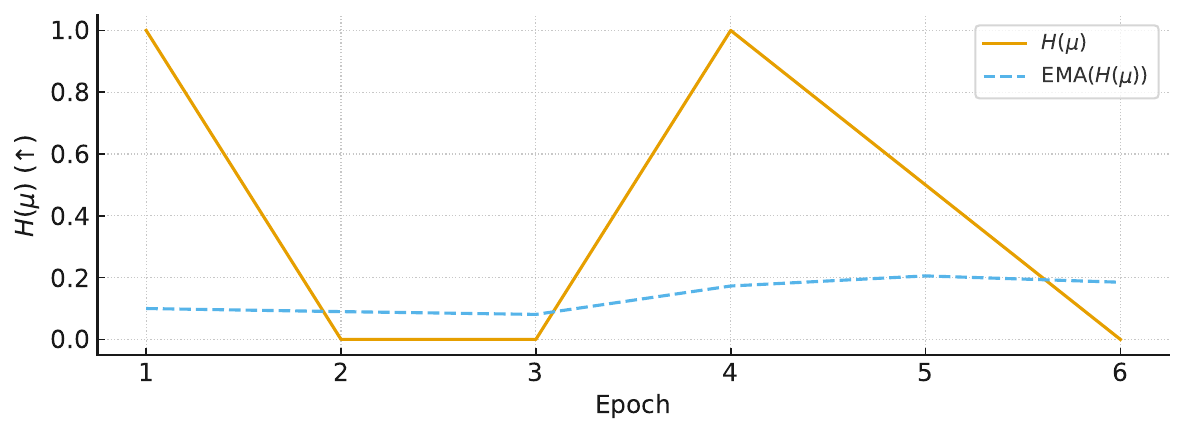}\\[-0.25em]
  \small (c) Membership entropy $H(\mu)$ (EMA) \hfill

  \caption{\textbf{Training dynamics under fixed compute}
  (a) Validation cross-entropy and its smoothed rate of change with phase bands
  (\,+RPA align,\,+Guardian,\,+SWA-select\,).
(b) Guardian’s $\tau_{\text{att}}$ adapts cautiously, avoiding over-tightening.
  (c) $H(\mu)$ rises then stabilizes, indicating non-collapsed, informative regimes.}
  \label{fig:dynamics}
\end{figure}

\label{sec:analysis}
\paragraph{Interpreting $B(T)$} Visualizing $B(T)$ early vs.\ late shows bands that mirror regime co-assignment: early blocks exhibit broad, low-contrast bands;
later blocks sharpen into interpretable stripes (local vs.\ long-range). Entropy floors keep memberships informative, preventing collapse to a constant bias.
\paragraph{Failure modes} (i) \emph{Early collapse.} If $H(\mu)$ collapses, RPA degenerates to near-constant bias;
entropy monitoring and a small positional mix mitigate this. (ii) \emph{Over-tightening.} Excessively low $\tau_{\text{att}}$ saturates heads;
the controller’s gain-shaped reward discourages this. (iii) \emph{Too few regimes.} Very small $R$ over-smooth $B(T)$;
modestly increasing $R$ or mixing a weak relative prior fixes this \rcite{shaw2018self}.

\paragraph{Practical guidance} Set $R\in[3,8]$; $K{=}R$; $\tau_{\text{align}}\in[0.6,0.8]$;
6–10 Sinkhorn iters; warm-in the bias; maintain a nonzero LR floor; start SWA only after entering a useful CE zone.
\paragraph{Design rationale} We wanted an attention prior that (i) is learned from the model’s own structure, not fixed, (ii) scales to variable lengths without retraining, and (iii) adds no inference parameters;
at test time we add a fixed, pre-computed additive bias to the attention logits, which does not change the asymptotic complexity or memory of the forward pass.
RPA satisfies (i) via $\mu$, (ii) via $\Phi(T)$, and (iii) because $B$ is a pre-softmax additive term \rcite{cuturi2013sinkhorn}.
Guardian targets the \emph{sign} of late-phase curvature: if marginal utility of sharpness is positive, it tightens;
otherwise it backs off \rcite{williams1992reinforce}. Selective SWA respects this asymmetry by averaging only during productive phases \rcite{izmailov2018swa}.
Finally, the context game complements RPA: by blending contexts according to a Nash mixture, the model observes the \emph{right} positional curves during training, making the learned RPA prior $B(T)$ more predictive at evaluation time across heterogeneous lengths.
\paragraph{Failure modes and diagnostics} If $\mu$ collapses early, RPA degenerates to a near-constant bias; monitoring $H(\mu)$ prevents this.
If Guardian over-tightens, heads saturate and CE rebounds; we detect this via a rise in saturation fraction and relax $\tau_{\text{att}}$.
When $R$ is too small, $B$ exhibits over-smooth bands that miss token-local structure;
increasing $R$ or mixing a small sinusoidal/relative prior fixes it \rcite{vaswani2017attention,shaw2018self}.
\section{Related Work}
Our prior relates to learned and relative/rotary position biases \rcite{vaswani2017attention,shaw2018self,su2021roformer,press2022alibi};
our fuzzy routing connects to MoE while avoiding brittle hard top-$k$ \rcite{shazeer2017moe,lepikhin2020gshard}; and fuzzy sets provide graded-membership foundations \rcite{zadeh1965fuzzy,mendel2001uncertain}.
The context-length mixture is a small population game trained by replicator/logit updates \rcite{taylor1978evolutionary,hofbauer1998evolutionary,sandholm2010population}.
Our analysis uses standard projected two-timescale stochastic approximation and ODE methods;
we adapt these to a single scalar control ($\tau$) with a shaped reward tailored to late-phase language modeling.
\section{Limitations}
\textbf{Scope} Single-task PoC (WT2) due to compute limits; time-series/equities loaders included only as templates. \textbf{Model scale} RPA’s benefits shrink as capacity and data grow stronger content logits. \textbf{Expressivity} Small $R$ induces low-rank priors that can underfit fine-grained structure. \textbf{Controller} Action space is intentionally narrow; richer per-head control may help but risks instability. \textbf{Overhead} We add no parameters but do not claim zero runtime overhead; we do not report inference micro-benchmarks.

\section{Core Mathematics (Practical Normalization)}
\label{sec:theory-details}
We formalize the entropic alignment underlying RPA.
Let $A\in[0,1]^{N\times K}$ solve
\[
A^*=\arg\min_{P\in\Delta_{N\times K}} \langle P,C\rangle - \varepsilon H(P)\quad
\text{s.t.}\quad P\mathbf{1}_K=\tfrac{1}{N}\mathbf{1}_N,\;\; P^\top\mathbf{1}_N=\tfrac{1}{K}\mathbf{1}_K.
\]
Define $B=AA^\top$.
\begin{proposition}[Row-sum and practical normalization]
Let $A\!\in\![0,1]^{N\times K}$ be the entropic OT alignment with
row marginals $A\mathbf{1}_K=\tfrac{1}{N}\mathbf{1}_N$ and column
marginals $A^\top\mathbf{1}_N=\tfrac{1}{K}\mathbf{1}_K$.
Define
$B = A A^\top$. Then for any position $i$,
\[
\sum_j B_{ij} \;=\; \frac{1}{NK}.
\]
Thus the per-query prior $\tilde B = (NK)\,B$ has row sums equal to $1$.
In practice (as in our implementation), we z-score $B$ and do not rely on
this exact row-sum;
optionally one can rescale the basis columns of $\Phi$
to enforce near-constant column sums and recover a constant row-sum in $\mu P \Phi^\top$.
\end{proposition}

\noindent\emph{Remark.} Because softmax is row-shift invariant and we subtract a rowwise max before the softmax, z-scoring $B$ (global mean/variance) preserves its shape while keeping its effective temperature commensurate with $QK^\top/\sqrt{d}$, which empirically stabilizes late-phase curvature and gradients.
\paragraph{Lemma (Shift/scale safety of a z-scored prior under softmax)}
Let $\widetilde{B}\in\mathbb{R}^{T\times T}$ be any additive attention prior and define the standardized prior $B=\mathrm{znorm}(\widetilde{B})=(\widetilde{B}-\mu)/\sigma$ with global mean $\mu$ and standard deviation $\sigma>0$, optionally followed by clipping to a bounded interval.
Consider pre-softmax logits $L=QK^\top/\sqrt{d}+B$ with a row-wise softmax. For any row-constant matrix $c\mathbf{1}^\top$, $\mathrm{softmax}(L+c\mathbf{1}^\top)=\mathrm{softmax}(L)$;
hence subtracting a global mean or adding any per-row constant does not change attention.
Moreover, scaling by a positive $\sigma^{-1}$ only rescales the relative contribution of the prior vs.\ content logits and is thus equivalent to adjusting a temperature on the prior.
\paragraph{Remark (Why we z-score in practice)}
The entropic-transport construction ensures a controlled row-sum structure for the raw prior $\widetilde{B}$, but training stability is governed by the dynamic range of $B$ relative to $QK^\top/\sqrt{d}$.
Z-scoring (and light clipping) makes the prior (i) zero-mean, preventing global logit drift, (ii) unit-variance, keeping its scale comparable to content similarity across lengths, and (iii) well-conditioned for gradient flow.
Because softmax is row-shift invariant and we also subtract a rowwise max before softmax in our implementation, z-scoring preserves the useful shape of $\widetilde{B}$ while avoiding unstable magnitudes;
the learned temperature $\tau_{\text{att}}$ then sets how strongly the prior should influence attention.
\section*{Reproducibility Statement}
We provide exact data preprocessing, tokenization, hyperparameters, schedules, seed handling, and compute disclosures.
The appendix contains \emph{complete, runnable listings} for the core pieces of code used in our runs.
These listings constitute the minimal reference implementation needed to reproduce our WT2 experiments.
We do \emph{not} release the full training harness (e.g., experiment orchestration, logging, or convenience utilities) during review to preserve anonymity;
no essential details are omitted for reproduction.

\section*{Ethics Statement}
Our experiments use public text data (WikiText-2) and do not involve human subjects, sensitive attributes, or personally identifiable information.
All authors read and adhered to the \href{https://neurips.cc/public/EthicsGuidelines}{NeurIPS Code of Ethics}.
We considered potential misuse risks: RPA is a structural bias; Guardian is a controller;
Our method adds no new inference parameters; the controller is disabled at inference;
at test time we add a fixed, pre-computed additive bias to the attention logits, which does not change the asymptotic complexity or memory of the forward pass, nor enables privacy attacks beyond standard Transformer baselines.
We disclose that LLM assistance was used to debug small code issues and refine text;
all design choices and analyses were made and verified by the authors.
\clearpage
\bibliographystyle{unsrtnat}
\bibliography{references}

\appendix
\section{Appendix}

\begin{table}[H]
\centering
\caption{Reference language-modeling numbers reported elsewhere on WikiText-2.
\emph{Not directly comparable} to our protocol (different training data/schedules/tokenizers or pretrained checkpoints).
Lower is better.}
\label{tab:ref-wt2}
\setlength{\tabcolsep}{6pt}
{\small
\begin{tabularx}{\linewidth}{@{}lccccY@{}}
\toprule
Model & Params (M) & Val CE $\downarrow$ & Test CE $\downarrow$ & PPL $\downarrow$ & Notes \\
\midrule
Fuzzy-Gated + RPA (ours) & $\sim$90 & \textbf{5.246} & -- & \textbf{189.8} & WT2 (raw-v1, GPT-2 BPE);
sequential, no-overlap eval \\
SFT Pythia--70M (HH) \rcite{sft_pythia70m_hh} & 70 & 5.195$\approx$ & -- & 180.27 & small SFT model;
WT2 word perplexity from model card (split unspecified); CE $\approx$ 5.19 \\
Pythia--70M (scratch, FP16) \rcite{supercompression_1bit_llm} & 70 & 4.298$\approx$ & -- & 73.1 & low-bit study FP16 reference;
CE $\approx$ 4.29 (protocol differs) \\
GPT-2 Large (pretrained) & 774 & 2.967$\approx$ & -- & 19.44 & WT2-raw-v1;
\emph{no overlap} (stride=1024). \emph{512 stride}: 16.44. \rcite{hf_perplexity_doc} \\
OPT-125M (baseline) & 125 & 2.744$\approx$ & -- & 15.55 & Raw WT2 baseline (Table 1, sparsity 0.0).
\rcite{opt_prune_tune_2024} \\
\bottomrule
\end{tabularx}
}
\vspace{0.25em}
\footnotesize\emph{Caveat:} Rows beyond “ours” come from different setups (tokenization, context handling, schedules, and/or pretraining), so they serve only as orientation.
\end{table}

\section{Proofs for Guardian Theory}\label{app:guardian-theory}
\begin{proof}[Proof sketch of Theorem~\ref{thm:guardian-converge}]
By Assumption~\ref{ass:guardian}(ii), the controller uses a slower stepsize than the optimizer, so on the $\tau$-timescale the weights $w_t$ track $w^*(\tau)$ (Assumption~\ref{ass:guardian}(iv)).
The projected update
$\tau_{t+1}=\Pi(\tau_t+\eta_t^\tau \hat g_t)$ with $\hat g_t$ an unbiased, bounded-variance estimate (Assumption~\ref{ass:guardian}(v)) is a stochastic approximation to the projected ODE $\dot \tau=\Pi(\nabla_\tau \bar R(\tau))$, where $\bar R(\tau)=\mathbb{E}[R(\tau;w^*(\tau))]$.
Boundedness and smoothness (Assumption~\ref{ass:guardian}(i,iii)) yield stability; the Robbins–Monro conditions imply almost-sure convergence to the ODE’s internally chain-transitive set via the ODE method.
If $\bar R$ is strictly concave on $[\tau_{\min},\tau_{\max}]$, that set is the singleton $\{\tau^\star\}$.
\end{proof}

\begin{proof}[Proof of Lemma~\ref{lem:guardian-onestep}]
Let $g=\nabla_\tau R(\tau;w)$ and let the step direction be $\delta$ with $\mathbb{P}[\operatorname{sign}(\delta)=\operatorname{sign}(g)] = p>\tfrac12$.
$L$-smoothness gives
$R(\tau{+}\alpha\delta;w) \ge R(\tau;w) + \alpha\, g\,\delta - \tfrac{L}{2}\alpha^2\|\delta\|^2$.
Taking expectation and using $\|\delta\|=1$ (w.l.o.g.) yields
$\mathbb{E}[R(\tau{+}\alpha\delta;w)-R(\tau;w)] \ge (2p{-}1)\alpha \|g\|
- \tfrac{L}{2}\alpha^2$.
The stated stepsize condition ensures the RHS is nonnegative.
\end{proof}

\begin{algorithm}[t]
\caption{Scripted temperature schedule (baseline)}
\label{alg:scripted-tau}
\begin{algorithmic}[1]
\State \textbf{Inputs:} $\tau_{\min},\tau_{\max},$ warmup steps $W,$ cosine floor $f\in(0,1)$, CE zone $z$, gain gate $\epsilon$
\For{step $t=1,2,\ldots$}
  \If{$t \le W$} \Comment{warm-in}
    \State $\tau \gets \tau_{\min} + (\tau_{\max}-\tau_{\min})\cdot t/W$
  \Else
    \State $u \gets \frac{t-W}{T-W}$,\quad $\tau \gets \tau_{\min} + (\tau_{\max}-\tau_{\min})\cdot \bigl(f + (1{-}f)\tfrac{1+\cos(\pi u)}{2}\bigr)$
  \EndIf
  \If{val-CE $< z$ \textbf{and} improvement$<\epsilon$} \Comment{avoid over-tightening}
    \State $\tau \gets \min\{\tau+\Delta,\ \tau_{\max}\}$ \quad \textit{(small nudge)}
  \EndIf
\EndFor
\end{algorithmic}
\end{algorithm}

\section{Full code listings}

\subsection{GaussianFuzzy (full)}\label{app:gaussianfuzzy}
\begin{lstlisting}[style=py,caption={Full GaussianFuzzy module.}]
class GaussianFuzzy(nn.Module):
    """Gaussian memberships μ_t over R regimes;
normalized to simplex."""
    def __init__(self, d_model: int, R: int, type2: bool = False):
        super().__init__()
        self.R = R
        self.proj = nn.Linear(d_model, d_model)
        self.centers = nn.Parameter(torch.randn(R, d_model) / math.sqrt(d_model))
        self.log_sigma = nn.Parameter(torch.zeros(R))
        self.type2 = type2
        if type2:
            self.uncert = 
nn.Sequential(nn.Linear(d_model, d_model), nn.ReLU(), nn.Linear(d_model, 1))

    def forward(self, h: torch.Tensor):
        z = self.proj(h)                              # [B,T,D]
        B, T, D = z.shape
        z2 = (z.unsqueeze(2) - self.centers.view(1, 1, self.R, D)).pow(2).sum(-1)
        inv2sig2 = torch.exp(-2 * self.log_sigma).clamp(1e-3, 1e3).view(1, 1, self.R)
      
  logits = torch.clamp(-0.5 * z2 * inv2sig2, -30.0, 30.0)
        mu = F.softmax(logits, dim=-1)
        mu = torch.nan_to_num(mu, nan=1.0 / self.R)
        if self.type2:
            u = torch.sigmoid(self.uncert(h))
            return mu, u
        return mu, None
\end{lstlisting}

\subsection{FuzzyMHA with RPA (full)}\label{app:fuzzymha}
\begin{lstlisting}[style=py,caption={Full FuzzyMHA with RPA and legacy bias.}]
class FuzzyMHA(nn.Module):
    def __init__(self, d_model, n_heads, dropout, 
R,
                 tau_att_init=1.0, pos_beta=0.2, kappa_init=0.5,
                 bias_clip=4.0, tau_max: float = 1.6,
                 use_rpa: bool = False, rpa_K: int = 0, tau_align: float = 0.7,
                 sinkhorn_iters: int = 8, rpa_posmix: float = 0.0, rpa_detach: bool = True):
    
    super().__init__()
        assert d_model % n_heads == 0
        self.d_model, self.n_heads, self.head_dim = d_model, n_heads, d_model // n_heads
        self.Wq = nn.Linear(d_model, d_model, bias=False)
        self.Wk = nn.Linear(d_model, d_model, bias=False)
        self.Wv = nn.Linear(d_model, d_model, bias=False)
        self.out_proj = nn.Linear(d_model, d_model, bias=False)
        self.value_gamma = nn.Linear(R, n_heads, bias=False)
       
 self.dropout = nn.Dropout(dropout)

        self.tau_att = nn.Parameter(torch.tensor(float(tau_att_init)))
        self.tau_max = float(tau_max)
        self.kappa = nn.Parameter(torch.tensor(float(kappa_init)))
        self.pos_beta = nn.Parameter(torch.tensor(float(pos_beta)))
        self.bias_clip = float(bias_clip)
        self.register_buffer("bias_scale", torch.tensor(1.0))

        # RPA controls
        self.use_rpa = bool(use_rpa)
        self.rpa_K = int(rpa_K) if int(rpa_K) > 0 else R
  
      self.tau_align, self.sinkhorn_iters = float(tau_align), int(sinkhorn_iters)
        self.rpa_posmix, self.rpa_detach = float(rpa_posmix), bool(rpa_detach)

    @staticmethod
    def _soft_blocks(T: int, K: int, device) -> torch.Tensor:
        t = torch.arange(T, device=device, dtype=torch.float32)
        c = torch.linspace(0, T - 1, K, device=device, dtype=torch.float32)
        w = max(1.0, (T / max(1, K)) * 1.5)
        dist = (t[:, None] - c[None, :]).abs() / w
   
     phi = 0.5 * (1.0 + torch.cos(torch.clamp(dist, 0, 1) * math.pi))
        phi = phi * (dist <= 1).float()
        phi = phi / (phi.sum(dim=1, keepdim=True) + 1e-6)
        return phi  # [T,K]

    @staticmethod
    def _pos_distance(T: int, device) -> torch.Tensor:
        i = torch.arange(T, device=device, dtype=torch.float32)
        return (i[:, None] - i[None, :]).abs() / max(1.0, T - 1.0)

 
   @staticmethod
    def _sinkhorn_knopp(scores: torch.Tensor, iters: int) -> torch.Tensor:
        X = scores
        for _ in range(iters):
            X = X / (X.sum(dim=1, keepdim=True) + 1e-9)
            X = X / (X.sum(dim=0, keepdim=True) + 1e-9)
        return X

    def _rpa_bias(self, mu: torch.Tensor) -> torch.Tensor:
        B, T, 
R = mu.shape
        Phi = self._soft_blocks(T, self.rpa_K, mu.device)        # [T,K]
        S = torch.einsum("btr,tk->brk", mu, Phi).mean(dim=0)     # [R,K]
        if self.rpa_detach: S = S.detach()
        Kmat = torch.exp(S / max(1e-6, self.tau_align)).clamp(1e-9, 1e9)
        P = self._sinkhorn_knopp(Kmat, self.sinkhorn_iters)      # ~ doubly-stochastic

        B_mat = torch.einsum("btr,rk,tk->btt", mu, P, Phi).mean(dim=0) 
 # [T,T]
        if self.rpa_posmix > 0.0:
            pos_bias = - self.pos_beta.clamp_min(0.0) * self._pos_distance(T, mu.device)
            B_mat = (1.0 - self.rpa_posmix) * B_mat + self.rpa_posmix * pos_bias

        B_mat = (B_mat - B_mat.mean()) / (B_mat.std() + 1e-6)
        B_mat = torch.nan_to_num(B_mat, nan=0.0, posinf=self.bias_clip, neginf=-self.bias_clip)
        tau = self.tau_att.clamp(0.6, self.tau_max)
       
 bias = torch.clamp((B_mat / tau).to(torch.float32), -self.bias_clip, self.bias_clip)
        return bias * self.bias_scale.clamp(0.0, 1.0)

    def _legacy_bias(self, mu: torch.Tensor) -> torch.Tensor:
        T = mu.size(1)
        pos_bias = - self.pos_beta.clamp_min(0.0) * self._pos_distance(T, mu.device)
        fuzz_sim = torch.einsum("btr,bsr->bts", mu, mu).mean(dim=0)
        fuzz_sim = (fuzz_sim - fuzz_sim.mean()) / (fuzz_sim.std() + 1e-6)
        curve = self.kappa.sigmoid() * fuzz_sim + (1.0 - self.kappa.sigmoid()) * pos_bias
  
      curve = torch.nan_to_num(curve, nan=0.0, posinf=self.bias_clip, neginf=-self.bias_clip)
        tau = self.tau_att.clamp(0.6, self.tau_max)
        bias = torch.clamp((curve / tau).to(torch.float32), -self.bias_clip, self.bias_clip)
        return bias * self.bias_scale.clamp(0.0, 1.0)

    def forward(self, x: torch.Tensor, mu: torch.Tensor, type2_u: Optional[torch.Tensor] = None):
        B, T, D = x.shape
        H, Hd = self.n_heads, self.head_dim
        q = self.Wq(x).view(B, T, H, Hd).transpose(1, 2)
 
       k = self.Wk(x).view(B, T, H, Hd).transpose(1, 2)
        v = self.Wv(x).view(B, T, H, Hd).transpose(1, 2)

        scores = torch.matmul(q, k.transpose(-2, -1)) / math.sqrt(Hd)
        bias = (self._rpa_bias(mu) if self.use_rpa else self._legacy_bias(mu)).unsqueeze(0).unsqueeze(0)
        scores = scores + bias

        mask = torch.triu(torch.ones(T, T, device=x.device, dtype=torch.bool), diagonal=1)
        scores = scores.masked_fill(mask, float('-inf'))
        
attn = self.dropout(torch.softmax(scores, dim=-1))

        out = torch.matmul(attn, v).transpose(1, 2).contiguous().view(B, T, D)
        out = self.out_proj(out)

        gamma = self.value_gamma(mu)                        # [B,T,H]
        gamma_s = torch.sigmoid(gamma.mean(dim=1, keepdim=True)).mean(dim=-1, keepdim=True)
        out = out * gamma_s
        return self.dropout(out), {"tau_att": self.tau_att.detach(), "kappa": self.kappa.detach()}
\end{lstlisting}

\subsection{Fuzzy 
Transformer block (full)}\label{app:fuzzyblock}
\begin{lstlisting}[style=py,caption={Full block: mem -> attn (RPA) -> fuzzy MoE.}]
class ExpertFFN(nn.Module):
    def __init__(self, d_model: int, d_ff: int, dropout: float):
        super().__init__()
        self.net = nn.Sequential(
            nn.Linear(d_model, d_ff), nn.GELU(), nn.Dropout(dropout),
            nn.Linear(d_ff, d_model), nn.Dropout(dropout)
        )
    def forward(self, x): return self.net(x)

class FuzzyMoE(nn.Module):
    def __init__(self, d_model: int, R: int, E: int = 4, 
top_k: int = 2, d_ff: int = 4, dropout: float = 0.1):
        super().__init__()
        self.E, self.top_k = E, top_k
        self.gate = nn.Linear(R, E, bias=False)
        self.experts = nn.ModuleList([ExpertFFN(d_model, d_model * d_ff, dropout) for _ in range(E)])
    def forward(self, x, mu):
        logits = torch.clamp(self.gate(mu), -30.0, 30.0)           # [B,T,E]
        
g = -torch.log(-torch.log(torch.rand_like(logits).clamp_(1e-6, 1-1e-6)))
        y = torch.nan_to_num(F.softmax((logits + g) / 0.5, dim=-1), nan=0.0)
        topk_vals, topk_idx = torch.topk(y, k=min(self.top_k, self.E), dim=-1)
        mask = torch.zeros_like(y).scatter_(-1, topk_idx, 1.0)
        y = (y * mask) / (y.sum(dim=-1, keepdim=True) + 1e-6)
        outs = [y[..., e].unsqueeze(-1) * expert(x) for e, expert in enumerate(self.experts)]
        out = torch.stack(outs, dim=-1).sum(-1)
        p 
= y.mean(dim=(0, 1))
        return out, {"lb_reg": ((p - 1.0 / self.E) ** 2).mean().detach(), "expert_usage": p.detach()}

class FuzzyTransformerBlock(nn.Module):
    def __init__(self, d_model, n_heads, dropout, R, d_ff_mult=4,
                 moe_E=4, moe_topk=2, type2=False, tau_att_init=1.0,
                 use_rpa=False, rpa_K=0, tau_align=0.7, sinkhorn_iters=8, rpa_posmix=0.0, rpa_detach=True):
        super().__init__()
        self.mem = GaussianFuzzy(d_model, R, type2=type2)
      
  self.attn = FuzzyMHA(d_model, n_heads, dropout, R, tau_att_init=tau_att_init,
                             use_rpa=use_rpa, rpa_K=rpa_K, tau_align=tau_align,
                             sinkhorn_iters=sinkhorn_iters, rpa_posmix=rpa_posmix, rpa_detach=rpa_detach)
        self.norm1, self.norm2 = nn.LayerNorm(d_model), nn.LayerNorm(d_model)
        self.moe = FuzzyMoE(d_model, R, E=moe_E, top_k=moe_topk, d_ff=d_ff_mult, dropout=dropout)
   
     self.res_gate = nn.Linear(R, 1, bias=False)
        self.dropout = nn.Dropout(dropout)
    def forward(self, x):
        mu, _ = self.mem(x)
        mu = torch.nan_to_num(mu, nan=1.0 / mu.size(-1))
        eta = torch.sigmoid(self.res_gate(mu))             # residual gate
        a_out, a_stats = self.attn(self.norm1(x), mu, None)
        x = x + eta 
* self.dropout(a_out)
        m_out, m_stats = self.moe(self.norm2(x), mu)
        x = x + eta * self.dropout(m_out)
        H_mu = - (mu * (mu.clamp_min(1e-8)).log()).sum(-1).mean()
        stats = {**a_stats, **m_stats, "mu_entropy": H_mu}
        return x, stats
\end{lstlisting}

\subsection{Guardian (full)}\label{app:guardian}
\begin{lstlisting}[style=py,caption={Full Guardian controller (policy + step/update).}]
@dataclass
class GuardianState:
    gate_delta: float
    sat_frac: float
    mu_entropy: float
    val_loss: float
    def to_tensor(self, device):
    
    return torch.tensor([self.gate_delta, self.sat_frac, self.mu_entropy, self.val_loss],
                            dtype=torch.float32, device=device)

class GuardianPolicy(nn.Module):
    def __init__(self, state_dim: int, action_dim: int):
        super().__init__()
        self.body = nn.Sequential(nn.Linear(state_dim, 64), nn.Tanh(), nn.Linear(64, 64), nn.Tanh())
        self.mean = nn.Linear(64, action_dim)
        self.log_std = nn.Parameter(torch.zeros(action_dim))
    def forward(self, s): z = 
self.body(s); return self.mean(z), torch.exp(self.log_std)
    def sample(self, s):
        m, std = self.forward(s);
a = m + std * torch.randn_like(m)
        logp = -0.5 * (((a - m) / (std + 1e-8)) ** 2 + 2 * self.log_std + math.log(2 * math.pi)).sum(dim=-1)
        return a, logp

class Guardian:
    def __init__(self, model, lr: float = 1e-3, enable: bool = True):
        self.model, self.enable = model, enable
        self.policy = GuardianPolicy(state_dim=4, action_dim=3)        # Δtau_att, Δλ_delta, Δλ_sat
       
 self.opt = torch.optim.Adam(self.policy.parameters(), lr=lr)
        self.lambda_delta, self.lambda_sat, self._last_logp, self.beta = 0.0, 0.0, None, 1.0
    def set_beta(self, beta: float): self.beta = float(beta)
    def get_tau(self): return torch.stack([blk.attn.tau_att for blk in self.model.blocks])
    def step(self, state: GuardianState) -> Dict[str, float]:
        if not self.enable:
            return {"lambda_delta": self.lambda_delta, "lambda_sat": self.lambda_sat,
                    "tau_att": float(self.get_tau().mean().item())}
  
      s = state.to_tensor(next(self.policy.parameters()).device).unsqueeze(0)
        a, logp = self.policy.sample(s);
self._last_logp = logp
        dtau, dl_delta, dl_sat = a[0].tolist();
scale = self.beta
        with torch.no_grad():
            for blk in self.model.blocks:
                tau = blk.attn.tau_att.data + 0.03 * scale * torch.tensor(dtau, device=blk.attn.tau_att.device)
                overshoot = torch.clamp(tau - blk.attn.tau_max, min=0.0)
                blk.attn.tau_att.data = (tau - 0.10 * overshoot).clamp(0.3, blk.attn.tau_max)
     
   self.lambda_delta = float(np.clip(self.lambda_delta + 0.01 * scale * dl_delta, 0.0, 1.0))
        self.lambda_sat   = float(np.clip(self.lambda_sat   + 0.01 * scale * dl_sat,   0.0, 0.6))
        return {"lambda_delta": self.lambda_delta, "lambda_sat": self.lambda_sat,
                "tau_att": float(self.get_tau().mean().item())}
    def update(self, reward: float):
        if not self.enable or self._last_logp is None: return
        loss = -(self._last_logp.mean() * torch.tensor(reward, dtype=torch.float32, device=self._last_logp.device))
        self.opt.zero_grad(); loss.backward();
self.opt.step()
\end{lstlisting}

\subsection{Chaos + training heuristics (full)}\label{app:chaos}
\begin{lstlisting}[style=py,caption={Chaos controller and in-loop heuristics.}]
class ChaosController:
    def __init__(self, r: float = 3.9, x0: float = 0.721, amp: float = 0.25, decay: float = 5e-4):
        self.r, self.x, self.amp0, self.decay, self.t, self._last = r, float(x0), float(amp), float(decay), 0, 1.0
    def _amp(self): return self.amp0 * math.exp(-self.decay * self.t)
    def step(self) -> float:
        self.x = self.r * self.x * (1.0 - self.x);
self.t += 1
        a = self._amp();
self._last = (1.0 - a) + a * self.x
        return self._last
    def factor(self) -> float: return self._last
    def temp(self, max_extra: float = 0.3) -> float: return 1.0 + max_extra * self._amp()

def apply_warm_in(model, step, warm_steps):
    scale = min(1.0, step / max(1, warm_steps))
    with torch.no_grad():
        for blk in model.blocks:
            blk.attn.bias_scale.fill_(scale)

def dropout_glide(model, base_p, phase):
    if base_p >= 0.08:
      
  tail = max(0.0, 1.0 - (phase / 0.60))
        p_drop = 0.08 + (base_p - 0.08) * tail
        for m in model.modules():
            if isinstance(m, torch.nn.Dropout): m.p = float(p_drop)
\end{lstlisting}

\subsection{LM micro-loss (full)}\label{app:lmloss}
\begin{lstlisting}[style=py,caption={Label-smoothed optimization + entropy floor with unsmoothed reporting.}]
def lm_loss_from_xy(cfg, model, x, y):
    logits, stats = model(x)                     # [B,T,V]
    V 
= logits.size(-1)
    flat_logits, flat_y = logits.view(-1, V), y.view(-1)
    ce_pure_sum = F.cross_entropy(flat_logits, flat_y, reduction="sum")
    ls = getattr(model, "_dyn_label_smooth", cfg.label_smooth)
    ce_sm_sum = F.cross_entropy(flat_logits, flat_y, label_smoothing=ls, reduction="sum")
    loss = ce_sm_sum
    if model.training and cfg.R >= 2:
        H_mu = stats.get("mu_entropy", None)
        if isinstance(H_mu, torch.Tensor):
            H_max = math.log(max(2, cfg.R))
            ent_pen = 
F.relu(cfg.ent_floor_eta * H_max - H_mu)
            lam_sat = getattr(model, "_lambda_sat", 0.0)
            loss = loss + (cfg.ent_floor_alpha * 0.5) * (1.0 + lam_sat) * ent_pen
    return loss, {"ce_pure_sum": ce_pure_sum.detach(),
                  "ce_sm_sum": ce_sm_sum.detach(),
                  "ntok": torch.tensor(y.numel(), device=ce_sm_sum.device)}, stats
\end{lstlisting}

\subsection{WT2 loaders + TS dataset (full)}\label{app:data}
\begin{lstlisting}[style=py,caption={WT2 loaders (HF/GPT-2 BPE) 
and sliding-window TS dataset.}]
def build_wikitext2_loaders(context_len: int, tokens_per_batch: int,
                            num_workers: int = DEFAULT_WORKERS, tokenizer_name: str = "gpt2"):
    datasets = _require("datasets");
transformers = _require("transformers")
    ds = datasets.load_dataset("wikitext", "wikitext-2-raw-v1")
    tok = transformers.AutoTokenizer.from_pretrained(tokenizer_name, use_fast=True)
    if tok.eos_token_id is None: tok.add_special_tokens({"eos_token": ""})
    eos_id = tok.eos_token_id
    def encode_split(split: str):
        texts = [t for t in ds[split]["text"] if t and not t.isspace()]
        ids = []
        for t in texts:
            ids.extend(tok.encode(t, add_special_tokens=False));
ids.append(eos_id)
        return torch.tensor(ids, dtype=torch.long)
    train_ids, val_ids, test_ids = map(encode_split, ("train","validation","test"))
    batch_size = max(1, tokens_per_batch // context_len)
    train_ds = RandomChunkDataset(train_ids, context_len)
    val_ds   = SequentialChunkDataset(val_ids, context_len)
    test_ds  = SequentialChunkDataset(test_ids, context_len)
    train_loader = DataLoader(train_ds, batch_size=batch_size, shuffle=True,  drop_last=True,
                              num_workers=num_workers, pin_memory=PIN_MEMORY)
    val_loader   
= DataLoader(val_ds,   batch_size=batch_size, shuffle=False, drop_last=True,
                              num_workers=num_workers, pin_memory=PIN_MEMORY)
    test_loader  = DataLoader(test_ds,  batch_size=batch_size, shuffle=False, drop_last=True,
                              num_workers=num_workers, pin_memory=PIN_MEMORY)
    return train_loader, val_loader, test_loader, tok.vocab_size, tok

class SlidingWindowTS(Dataset):
    def __init__(self, df: pd.DataFrame, input_cols: List[str], target_cols: 
Optional[List[str]],
                 context_len: int, horizon: int, stride: int = 1, normalize: bool = True):
        self.X = df[input_cols].astype(np.float32).values
        self.Y = self.X if target_cols is None else df[target_cols].astype(np.float32).values
        self.context_len, self.horizon, self.stride = context_len, horizon, stride
        mu = self.X.mean(0, keepdims=True) if normalize else 0.0
        sigma = self.X.std(0, keepdims=True) + 1e-6 if normalize else 
1.0
        self.Xn = (self.X - mu) / sigma
        if target_cols is None:
            self.Yn = (self.Y - mu) / sigma
        else:
            ymu, ysig = self.Y.mean(0, keepdims=True), self.Y.std(0, keepdims=True) + 1e-6
            self.Yn = (self.Y - ymu) / ysig
        self.N = len(self.Xn)
 
   def __len__(self): return max(0, (self.N - (self.context_len + self.horizon)) // self.stride)
    def __getitem__(self, idx: int):
        i = idx * self.stride
        x = self.Xn[i:i + self.context_len]
        y = self.Yn[i + self.context_len:i + self.context_len + self.horizon]
        return torch.from_numpy(x), torch.from_numpy(y)
\end{lstlisting}

\end{document}